\documentclass[conference]{IEEEtran}
\IEEEoverridecommandlockouts

\usepackage{cite}
\usepackage{amsmath,amssymb,amsfonts}
\usepackage{algorithmic}
\usepackage{algorithm}
\usepackage{graphicx}
\usepackage{textcomp}
\usepackage{xcolor}
\usepackage{bm}
\usepackage{amsthm}

\newtheorem{theorem}{Theorem}
\newtheorem{lemma}[theorem]{Lemma}
\newtheorem{definition}[theorem]{Definition}

\newtheorem{example}[theorem]{Example}

\newcommand{\E}{\mathbb{E}}

\def\BibTeX{{\rm B\kern-.05em{\sc i\kern-.025em b}\kern-.08em
    T\kern-.1667em\lower.7ex\hbox{E}\kern-.125emX}}

\begin{document}

\title{DCM Bandits: Multiplayer Information Asymmetric Cascading Bandits for Multiple Clicks}

\author{\IEEEauthorblockN{Andy Wang\thanks{$^{*}$These authors contributed equally.}$^{*}$}
\IEEEauthorblockA{\textit{Department of Computer Science}\\
\textit{University of California, Los Angeles}\\
Los Angeles, USA\\
dxwang@ucla.edu}
\and
\IEEEauthorblockN{Charlton Shih$^{*}$}
\IEEEauthorblockA{\textit{Department of Computer Science}\\
\textit{University of California, Los Angeles}\\
Los Angeles, USA\\
charltonshih645@g.ucla.edu}
\and
\IEEEauthorblockN{William Chang}
\IEEEauthorblockA{\textit{Department of Mathematics}\\
\textit{University of California, Los Angeles}\\
Los Angeles, USA\\
chang314@g.ucla.edu}
}

\maketitle

\begin{abstract}
In this work, we extend the Dependent Click Model (DCM) Bandits to a multiplayer information-asymmetric setting, where multiple agents interact with a shared ranked list and may observe multiple clicks per session, introducing new challenges for selection strategies. We study asymmetry in (1)~actions and (2)~rewards, providing sublinear regret guarantees for three settings where at least one asymmetry is present. Establishing matching information-theoretic lower bounds for these settings is left as an open problem. We further show that for small termination probabilities, the termination ranking need not be known, improving on prior single-agent results. Experiments confirm that our algorithms perform well across asymmetric environments, and highlight the critical role of feedback structure---full versus first-click---in coordinating exploration and minimizing regret.
\end{abstract}

\begin{IEEEkeywords}
multi-armed bandits, cascading bandits, multiplayer learning, information asymmetry, dependent click model, online learning
\end{IEEEkeywords}

\section{Introduction}

The cascade model~\cite{kveton2015cascading} and its dependent-click variant (DCM)~\cite{guo2009efficient, katariya2016dcm} are now standard frameworks for ranking under bandit feedback. In DCM, a user examines a ranked list, may click on multiple items, and terminates the session at any point governed by slot-specific satisfaction probabilities. The literature has produced regret-optimal single-agent algorithms~\cite{kveton2015cascading, katariya2016dcm, zong2016cascading, pmlr-v48-lif16, zhong2021thompson} and extensions to combinatorial actions~\cite{kveton2015combinatorial, fourati2024combinatorial}, user fatigue~\cite{cao2020fatigue}, exposure bias~\cite{mansoury2024mitigating}, and adversarial robustness~\cite{xie2025adversarialcascade}.

Modern recommender systems, however, are rarely single-agent: industrial deployments involve multiple modules or stakeholders that jointly determine what is shown, each with partial state and partial feedback~\cite{zou2025surveyrealworldrecommendersystems, wang2025recommending}. This motivates a \emph{decentralized} view of cascading bandits in which several players contribute distinct components of the joint ranking and observe different parts of the feedback. The multiplayer bandit literature~\cite{boursier2022survey, branzei2021cooperation, chang2022online, chang2023optimal, chang2025multiplayerinformationasymmetriccontextual, chang2025metric} addresses coordination under such constraints, but assumes simple per-action rewards; the click-model literature, conversely, assumes a single decision-maker. The intersection has, to our knowledge, not been studied.

\paragraph{Our Contribution.}
We introduce \emph{multiplayer cascading bandits with multi-click feedback} and develop algorithms for three settings of information asymmetry:
\begin{itemize}
  \item \textbf{Problem A (action asymmetry).} \texttt{mCascadeUCB}: a decentralized UCB algorithm in which players jointly select the top-$K$ items by upper confidence bound and update shared statistics from observed cascade feedback.
  \item \textbf{Problem B (reward asymmetry).} \texttt{mCascadeUCB-Intervals-Ranking}: an elimination algorithm in which separated confidence intervals implicitly signal suboptimality through action choices. A round-robin variant removes assumptions required in~\cite{katariya2016dcm} when termination probabilities are small.
  \item \textbf{Problem C (both).} \texttt{mMDSEE-TopK}: a phased explore-then-commit algorithm requiring no communication, extending~\cite{chang2022online, chang2023optimal} to multi-slot cascade feedback.
\end{itemize}

We prove sublinear regret for all three algorithms and validate them empirically. To our knowledge this is the first systematic study of decentralized multi-click cascading bandits. Sections~\ref{sec:problem}--\ref{sec:experiments} present the model, algorithms with regret analyses, and experiments; Section~\ref{sec:related} discusses related work.

\section{Problem Statements}\label{sec:problem}

\subsection{Single-Agent Cascading Bandits}

We first review the classical single-agent cascading bandit model, which captures user interaction with ranked lists under the cascade hypothesis: users scan items sequentially and click the first attractive item, if any.

Let $E$ be a ground set of $L$ items. In each round $t=1,\dots,T$, the agent recommends an ordered list $A_t=(e_1,\dots,e_K)$ of $K$ distinct items from $E$. Each item $e\in E$ has an unknown attraction probability $w(e)\in[0,1]$, representing the probability that a user clicks on the item when examined. Attraction probabilities are assumed independent and stationary.

The agent receives reward $1$ if any item is clicked and $0$ otherwise. The objective is to learn the optimal ranking $\bm{e^\star}=(e^\star_1,\dots,e^\star_K)$ maximizing the probability that at least one item is clicked:
\begin{equation}
\bm{e^\star} =
{\arg\max}_{A \in \Pi_K(E)}
\left(1 - \prod_{k=1}^K (1-w(\bm{e}_k))\right),
\end{equation}
where $\Pi_K(E)$ denotes the set of ordered $K$-subsets of $E$.

\subsection{Multi-Agent Cascading Bandits}

We now consider a decentralized multi-agent extension with $M$ players. Each player $i\in\{1,\dots,M\}$ has a local action set $E^i$ with $|E^i|=L$.

In each round, player $i$ selects $K$ local items from $E^i$. The players' choices combine to form a joint ranking $A_t=(\bm{a}_1,\dots,\bm{a}_K)$, where each joint item $\bm{a}_k=(a_k^1,\dots,a_k^M)$ contains one action from each player. Thus player~$i$ controls the $i$-th coordinate $a_k^i\in E^i$ of each joint item.

This setting captures two key features: (i)~\textbf{local autonomy}, where each player controls only their own marginal choice from $E^i$; and (ii)~\textbf{global coordination}, where the joint ranking emerges from the simultaneous choices of all players.

Feedback remains partial as in the cascade model: agents observe the clicked joint item (if any), which reveals that earlier items were unattractive while later items remain unobserved. The reward is shared across players and equals~$1$ if any item is clicked.

To maintain decentralization, agents may agree on a protocol before learning begins but cannot communicate during execution. The joint action space is large: there are $L^M$ joint items and $\frac{(L^M)!}{(L^M-K)!}$ possible rankings of size~$K$.

We consider three information structures:

\paragraph{Problem A: Action asymmetry.}
Players cannot observe other players' actions, but all players observe the same click feedback.

\paragraph{Problem B: Reward asymmetry.}
Players observe each other's actions but receive independent (i.i.d.) click feedback.

\paragraph{Problem C: Action and reward asymmetry.}
Players receive i.i.d.\ click feedback and cannot observe the actions of other players.

\subsection{Multi-Agent Cascading Bandits with Multiple Clicks}

We extend the model by introducing slot-dependent termination probabilities. In addition to attraction probabilities, each slot $j$ has a termination probability $v_j$ with $1\le j\le K$.

When a user clicks an item at position $j$, the session terminates with probability $v_j$; otherwise the user continues examining later items. The termination probability is independent of the item and the round.

As before, feedback is partial: items before the click are observed to be unattractive, the clicked item is attractive, and later items remain unobserved. Under this model the optimal ranking becomes
\begin{equation}
\bm{e^\star} =
{\arg\max}_{A \in \Pi_K(E)}
\left(1 - \prod_{j=1}^K (1-v_j w(\bm{e}_j))\right).
\end{equation}

The termination probabilities allow multiple clicks within a session, increasing the amount of feedback per round while also introducing additional uncertainty for learning and coordination.

\section{Main Results}\label{sec:main}

\subsection{Problem A: Information Asymmetry in Actions}

We now consider Problem~A, where players observe identical rewards but cannot observe one another's actions. The main challenge is that players must infer which joint item was played in order to correctly update empirical statistics. Otherwise, \emph{miscoordination} may occur, where different players update different joint items despite observing the same reward.

Because all players observe identical clicks, coordination can be achieved through a predetermined exploration phase. Specifically, during the first $K_{\max}=K_1\cdots K_M$ rounds players follow a fixed joint schedule that explores all joint items. This ensures that all players obtain identical statistics for every joint item.

After this phase, players compute UCB indices for each joint item:
\begin{equation}\label{eq:ucb}
   \operatorname{UCB}_t(\bm{e}) =
   \begin{cases}
   \infty & \text{if } n_{t-1}(\bm{e}) = 0, \\
   \widehat{w}_{n_{t-1}(\bm{e})}(\bm{e}) + c_{n_{t-1}(\bm{e})} & \text{otherwise.}
   \end{cases}
\end{equation}

Since all players share identical estimates, they select the same joint item at each round by choosing the $K$ items with the largest UCB values. Ties are broken using a fixed lexicographic order on~$\mathbb{R}^M$ from~\cite{chang2022online}.

\begin{definition}
We say $(x_1,\dots,x_M) < (y_1,\dots,y_M)$ if there exists $n$ such that $x_i=y_i$ for all $i<n$ and $x_n<y_n$.
\end{definition}

Coordination then follows by induction: if players select the same actions up to time $t-1$, identical UCB statistics imply they select the same joint action at time~$t$.

Unlike the classical cascade model, we allow multiple clicks and update all observed prefixes up to $\min(C_t,K)$. Because all players maintain identical statistics, no communication is required to maintain coordination. The procedure is given in Algorithm~\ref{alg:mUCB-short}.

\begin{algorithm}[t]
  \caption{\texttt{mCascadeUCB-A}}
  \label{alg:mUCB-short}
  \begin{algorithmic}[1]
    \REQUIRE $M$ players with action sets $E^i$ of size $L$, horizon $T$, list size $K$
    \STATE $\mathbf{E}\gets E^1\times\cdots\times E^M$
    \STATE Initialize $\widehat w(\mathbf e)\gets0$, $n(\mathbf e)\gets0$ for all $\mathbf e\in\mathbf E$
    \FOR{$t=1$ \textbf{to} $T$}
      \STATE Compute $\mathrm{UCB}_t(\mathbf e)$ for all $\mathbf e\in\mathbf E$
      \STATE Select top-$K$ items $(\mathbf e_1,\dots,\mathbf e_K)$ by $\mathrm{UCB}_t$
      \STATE Players extract their components and observe click $C_t$
      \FOR{$j=1$ \textbf{to} $\min\{C_t,K\}$}
        \STATE Update $\widehat w$ and $n$ for $\mathbf e_j$
      \ENDFOR
    \ENDFOR
  \end{algorithmic}
\end{algorithm}

\begin{theorem}\label{thm:mUCB}
If each player uses \texttt{mCascadeUCB-A} under Problem~A, the expected $T$-step regret satisfies
\[
R_T \leq \sum_{\bm{e}=\bm{K+1}}^{\bm{L^M}}
\frac{12}{\Delta_{\bm{e},\bm{K}}}\log T
+ \frac{\pi^2}{3}L^M.
\]
\end{theorem}

This matches the regret order of the single-agent cascading UCB algorithm. The exponential dependence on~$L^M$ is unavoidable for the algorithm as stated: combinatorial cascading bandits over factored action sets~\cite{kveton2015combinatorial, fourati2024combinatorial} obtain polynomial regret only when reward is a known function (e.g.\ disjunctive) of per-item attractions, whereas we make no parametric assumption tying $w(\bm{e})$ to per-player marginals $w^i(a^i)$---coordinates may interact arbitrarily, as with complementary product bundles. Combining a known link function with our coordination mechanism to obtain polynomial-in-$L$ regret is a natural extension that we leave to future work.

\subsection{Problem B: Information Asymmetry in Rewards}

We now consider Problem~B, where players observe the same actions but receive independent (i.i.d.) click realizations. As a result, players maintain different reward estimates and UCB indices for the same joint item, which may lead to coordination failures. The cascading feedback structure further complicates learning since items receive different numbers of observations across players.

To address this, we use confidence intervals with both upper and lower bounds. For each joint item $\bm{e}$, players compute
\begin{align*}
\mathrm{UCB}_t(\bm{e})&=\widehat{w}_{n_{t-1}(\bm{e})}+c_{t-1,n_{t-1}(\bm{e})},\\
\mathrm{LCB}_t(\bm{e})&=\widehat{w}_{n_{t-1}(\bm{e})}-c_{t-1,n_{t-1}(\bm{e})}.
\end{align*}

Standard concentration inequalities (Lemma~\ref{Corollary5.5}) imply that the true click probability lies in $(\mathrm{LCB}_t(\bm{e}),\mathrm{UCB}_t(\bm{e}))$ with high probability. If $\mathrm{UCB}_t(\bm{e}) < \mathrm{LCB}_t(\bm{e}')$, the algorithm can safely eliminate~$\bm{e}$.

\paragraph{mCascadeUCB-Intervals-Ranking.}
The algorithm proceeds in $K$ phases, identifying one item per phase. At phase~$r$, the candidate set is $\mathcal{D}=E\setminus\mathcal{R}$, where $\mathcal{R}$ contains items selected in earlier phases.

Players cycle through items in $\mathcal{D}$ using a shared ordering and repeatedly place the current candidate in the first slot. Because players observe independent clicks, their confidence intervals may differ. Coordination is achieved through \emph{implicit signaling}: if a player detects that $\mathrm{UCB}_t(\bm{e}) < \mathrm{LCB}_t(\bm{e}')$, they \emph{sabotage} $\bm{e}$ by deviating from the scheduled marginal action. This observable deviation signals that $\bm{e}$ should be removed from $\mathcal{D}$.

As intervals shrink, inferior items are eliminated until a single item remains, which is appended to $\mathcal{R}$. Repeating this process for $K$ phases identifies the full ranking.

\begin{example}
Let $K=2$ with two players. A joint item $(i,j)$ denotes player~1 choosing $i$ and player~2 choosing $j$. Suppose the candidate cycle is $(1,1)\to(1,2)\to(2,1)\to(1,1)\to\cdots$. The corresponding placements are
\begin{align*}
t&=1{:}\ ((1,1),(1,2)),\\
t&=2{:}\ ((1,2),(2,1)),\\
t&=3{:}\ ((2,1),(1,1)),\dots
\end{align*}
If a player detects $\mathrm{UCB}_t((1,2)) < \mathrm{LCB}_t((1,1))$, they sabotage round $t\!=\!2$ by playing $((1,1),(2,1))$ instead of the scheduled $((1,2),(2,1))$. Other players detect the deviation and remove $(1,2)$ from the cycle, leaving $(1,1)\to(2,1)\to(1,1)\to\cdots$.
\end{example}

This procedure achieves gap-dependent regret $O(\log T)$. Note that obtaining sublinear regret requires recovering not only the top-$K$ set but also their correct ordering. Sabotage signaling is unambiguous because the round-robin schedule is deterministic and common knowledge: every player can compute the scheduled joint action from $(t, \mathcal{R}, \mathcal{D})$ alone, so any deviation is a signal rather than noise. The only remaining failure mode---a player eliminating an item whose true attraction is higher than the eliminator's---is controlled by the confidence radius of Lemma~\ref{Corollary5.5} and contributes only an $O(L^M/T)$ additive term to the regret, handled in the proof of Theorem~\ref{thm:mCascadeUCB-Intervals-Ranking}.

\begin{algorithm}[t]
\caption{\texttt{mCascadeUCB-Intervals-Ranking-multiple}}
\label{alg:mCascadeUCB-Intervals-Ranking-multiple-short}
\begin{algorithmic}[1]
\REQUIRE $M$ players, item set $E$, horizon $T$, target size $K$
\STATE $\mathcal{R}\gets\emptyset$; initialize means and counts
\FOR{$r=1$ \textbf{to} $K-|\mathcal{R}|$}
  \STATE $\mathcal{D}\gets E\setminus\mathcal{R}$; set cycle pointer $c\gets 0$
  \WHILE{$|\mathcal{D}|>1$}
    \STATE Place items $\mathcal{D}[c],\mathcal{D}[c{+}1],\ldots,\mathcal{D}[c{+}K{-}1]$ (mod $|\mathcal{D}|$) in slots $1,\ldots,K$
    \STATE Compute $\mathrm{UCB}_t,\mathrm{LCB}_t$ for $e\in\mathcal{D}$
    \IF{$\exists\,e\in\mathcal{D}$ with $\mathrm{UCB}_t(e) < \max_{e'\in\mathcal{D}}\mathrm{LCB}_t(e')$}
      \STATE Sabotage scheduled placement by deviating; signal elimination of $e$
    \ENDIF
    \STATE Observe cascade; for each slot $k\le\min(C_t,K)$ update estimate of the item placed there
    \STATE Remove dominated items from $\mathcal{D}$; advance $c\gets c+1$
  \ENDWHILE
  \STATE Append remaining item to $\mathcal{R}$
\ENDFOR
\STATE Recommend $\mathcal{R}$ ordered by empirical means
\end{algorithmic}
\end{algorithm}

\begin{theorem}\label{thm:mCascadeUCB-Intervals-Ranking}
Let $\bm{e}^\star_1,\dots,\bm{e}^\star_K$ denote the top-$K$ joint items (in decreasing order of attraction probability), and for each $r\in\{1,\dots,K\}$ let $\mu^\star_r = w(\bm{e}^\star_r)$ and $\Delta_{\bm{e}^\star_r,\bm{e}} = \mu^\star_r - w(\bm{e})$ for any suboptimal $\bm{e}$ in the $r$-th elimination phase. The regret of \texttt{mCascadeUCB-Intervals-Ranking} under Problem~B satisfies
\[
R_T \leq \sum_{r=1}^{K}\;\;
\sum_{\bm{e}\,:\,w(\bm{e}) < \mu^\star_r}
\frac{(4+4\sqrt{2})^2\,\log T}{\Delta_{\bm{e}^\star_r,\,\bm{e}}^{\,2}}.
\]
\end{theorem}

\paragraph{Round-Robin Ranking with Multiple Placements.}
The previous algorithm only uses feedback from the first slot. When termination probabilities are small, additional slots can provide useful observations. The algorithm \texttt{mCascadeUCB-Intervals-Ranking-multiple} (Algorithm~\ref{alg:mCascadeUCB-Intervals-Ranking-multiple-short}) exploits this by cycling through candidates and placing $K$ items each round.

If the upper confidence bound of an item falls below another item's lower confidence bound, a player sabotages the scheduled placement, signaling elimination. An item $e$ placed in slot $k$ is \emph{observed} if the cascade reaches that slot. The observation probability for slot $k$ is
\[
p_k(e)=\prod_{j=1}^{k-1}\Big[(1-w(\bm{e}_j))+w(\bm{e}_j)v_j\Big].
\]
Define the minimum slot-survival probability
\[
p_{\min}=
\min_{1\le k\le K}\;
\min_{(\bm{e}_1,\ldots,\bm{e}_{k-1})}
\prod_{j=1}^{k-1}
\Big[(1-w(\bm{e}_j))+w(\bm{e}_j)v_j\Big].
\]

\begin{lemma}\label{lemma:observations}
If a joint item $\bm{e}$ is placed in slot $j>1$ a total of $M$ times, then with probability at least $1-\frac{1}{T}$ it receives at least
$\left(1-\sqrt{\frac{\log T}{2p_{\min}^2 M}}\right)Mp_{\min}$
observations. In particular, the bound is non-vacuous whenever
\[
M \;\ge\; \frac{\log T}{2\,p_{\min}^2},
\]
which is the regime in which the multi-placement strategy yields strictly more information than first-slot-only exploration.
\end{lemma}

\begin{theorem}\label{thm:mCascadeUCB-Intervals-Ranking-multiple}
The regret of \texttt{mCascadeUCB-Intervals-Ranking-multiple} satisfies
\[
R_T = O\!\left(\sum_{\bm{e}}L^M \left(
\frac{
B + \sqrt{B^2 + 4A \cdot R_{\bm{e}}}
}{2A}
\right)^{\!2}\right),
\]
where $A = 1+(K-1)p_{\min}$, $B = (K-1)\sqrt{\frac{\log T}{2}}$, and $R_{\bm{e}} = \frac{8\log T}{\Delta_{\bm{e}}^2}+1$.
\end{theorem}

When $p_{\min}\approx1$ (i.e., termination probabilities are small), the bound approaches the classical cascading-bandit regret $O(\log T/\Delta^2)$~\cite{kveton2015cascading,katariya2016dcm}.

\subsection{Problem C: Information Asymmetry in Both}

We now consider Problem~C, which combines the difficulties of Problems~A and~B: players cannot observe one another's actions \emph{and} they receive asymmetric (i.i.d.) click feedback. As a result, \texttt{mCascadeUCB-Intervals} from Problem~B cannot coordinate eliminations without action observability, and the approach for Problem~A fails because players no longer share identical reward statistics.

To handle both asymmetries, we adapt the MDSEE framework~\cite{chang2022online} to the multi-agent cascade setting. Our goal is to identify the top-$K$ \emph{joint} items using a phased explore-then-commit strategy. The resulting algorithm, \texttt{mMDSEE-TopK}, alternates between exploration and exploitation phases according to a schedule $F(\lambda)$ (default $F(\lambda)=\lambda$), where $\lambda$ is the phase index.

In the exploration phase, the algorithm spends $L\cdot F(\lambda)$ rounds uniformly sampling items via a fixed round-robin schedule: at each round~$t$ it recommends a $K$-list so that each item appears in each slot sufficiently often. The user returns a click position $C_t\in\{1,\ldots,K,\infty\}$. For each observed position $k\le C_t$, the algorithm updates $n_t(\bm{e})$ and $\widehat{w}_t(\bm{e})$ for the item in slot $a_t^k$.

After exploration, the exploitation phase forms a ranking $\mathcal{R}_t$ of the $K$ items with largest empirical means $\widehat{w}_t(\bm{e})$ and repeatedly recommends $\mathcal{R}_t$ until the next power-of-two boundary $2^\tau$ (smallest $\tau$ with $2^\tau>t$). The phase index $\lambda$ is then incremented. We only update estimates using exploration rounds: under asymmetric rewards and unobservable joint actions, feedback during exploitation cannot be reliably attributed to specific joint items.

The algorithm is given in Algorithm~\ref{alg:mMDSEE-TopK-short}.

\begin{algorithm}[t]
  \caption{\texttt{mMDSEE-TopK}}
  \label{alg:mMDSEE-TopK-short}
  \begin{algorithmic}[1]
    \REQUIRE Item set $E$ of size $L$, target size $K$, horizon $T$, schedule $F(\lambda)=\lambda$
    \ENSURE Top-$K$ items identification
    \STATE Initialize $\hat{w}_0(e) \gets 0$, $n_0(e) \gets 0$ for all $e \in E$
    \STATE $\lambda \gets 1$; \quad $t \gets 1$
    \WHILE{$t < T$}
      \STATE \textbf{Exploration:} For $F(\lambda)\cdot L$ rounds, cycle through items with a sliding window of size $K$
      \FOR{$j=1$ \textbf{to} $L$}
        \STATE Recommend $[\,e_j,\, e_{(j+1)\bmod L},\, \ldots,\, e_{(j+K-1)\bmod L}\,]$
        \STATE Observe click $C_t$ and update estimates
        \STATE $t \gets t+1$
      \ENDFOR
      \STATE \textbf{Exploitation:} Let $\mathcal{R}_t$ be the top-$K$ items by $\hat{w}_t(e)$
      \STATE $T_{\text{exploit}} \gets 2^r - t$ \quad (with $2^r > t$)
      \FOR{$j=1$ \textbf{to} $T_{\text{exploit}}$}
        \STATE Recommend $\mathcal{R}_t$, observe feedback
        \STATE $t \gets t+1$
      \ENDFOR
      \STATE $\lambda \gets \lambda+1$
    \ENDWHILE
  \end{algorithmic}
\end{algorithm}

\paragraph{Regret with First-Slot Feedback.}
We first analyze the special case where players observe only whether the first slot terminates.

\begin{theorem}[First-Slot Feedback]\label{thm:mDSEE-TopK-1}
Consider Algorithm~\ref{alg:mMDSEE-TopK-short}. Suppose that $\varepsilon < \tfrac{1}{2}\min_{1 \le i \le K} \min_{\bm e} \Delta_{i+1,i}$, and let $F(\lambda)=\lambda$. Then the cumulative regret after $T$ rounds satisfies
\begin{align*}
R_T &= O\!\Big(
L^M \log \log T \, \log T \\
&\quad+ \frac{M L^M \sqrt{\pi}}{\sqrt{2c}\,\varepsilon}
e^{\frac{1}{8c\varepsilon^2}}
\left[
1 + \mathrm{erf}\!\left(\frac{1}{\sqrt{8c}\,\varepsilon}\right)
\right]
\Big).
\end{align*}
\end{theorem}

Here, $\Delta_{i+1,i}$ denotes the gap between the $(i+1)$-th and $i$-th best \emph{joint} items. Choosing $\varepsilon = \tfrac{1}{2}\min_{1 \le i \le K} \Delta_{i+1,i}$ is sufficient to separate optimal from suboptimal items and to recover the correct ordering within the top-$K$.

The first term reflects exploration across phases; since phase lengths grow exponentially, the effective sample size increases rapidly and error probabilities decay quickly. The second term accounts for mis-commitment: the probability that players commit to an incorrect joint item due to estimation error. We use $F(\lambda)=\lambda$ to obtain sharper asymptotic bounds than~\cite{chang2022online}.

\paragraph{Regret with Multiple-Slot Feedback.}
We next consider the case where players use feedback from all observed slots up to~$K$. The algorithm is unchanged; only the exploration updates incorporate additional observed joint items, improving estimation without altering coordination.

\begin{theorem}[Multi-Slot Feedback]\label{thm:mDSEE-TopK-2}
Under the same setting and notation as Theorem~\ref{thm:mDSEE-TopK-1}, let $c = 1 + (K-1)(1 - \frac{\log T}{2p_{\min}^2 F(T)})p_{\min}$, with $\alpha = 1 + (K-1)p_{\min}$. If $F(\lambda) = \lambda$, then
\begin{align*}
R_T &= O\!\Big(
\log T \log\log T \\
&\quad+ \frac{M L^M \sqrt{\pi}}{\sqrt{2c}\,\varepsilon\sqrt{\alpha}}
e^{\frac{1}{8c\alpha\varepsilon^2}}
\left[
1 + \mathrm{erf}\!\left(\frac{1}{\sqrt{8c\alpha}\,\varepsilon}\right)
\right]
\Big).
\end{align*}
\end{theorem}

The factor $\alpha$ captures the benefit of multi-slot feedback: observing more slots increases the effective number of observations per round by $\alpha=1+(K-1)p_{\min}$, tightening concentration and reducing regret. When $p_{\min}=0$, no additional slots are observed and we recover Theorem~\ref{thm:mDSEE-TopK-1} with $c=1$. As $p_{\min}\to 1$, we obtain $c\approx K$, corresponding to nearly $K$ effective observations per round.

\section{Experiments}\label{sec:experiments}

We evaluate our algorithms under the DCM cascade simulator: at slot $k$ with item $e$, the user clicks with probability $w(e)$; if a click occurs, the session terminates with probability $v_k$. The terminating click position $C_t$ (or $\infty$) is the observed feedback. Reward is $\mathbb{I}\{C_t \neq \infty\}$ and instantaneous regret is $\bigl[1 - \prod_k(1 - v_k w^\star_k)\bigr] - \bigl[1 - \prod_k(1 - v_k w(\bm{e}^t_k))\bigr]$, where $w^\star_1,\dots,w^\star_K$ are the top-$K$ attractions paired with the largest $v_k$ values (optimal by the rearrangement inequality).

\paragraph{Baselines.}
To verify that joint-arm UCB is not artificially inflating measured regret, we compare against an \emph{independent per-player UCB} baseline in which each player runs a single-agent UCB over its local item set $E^i$, ignoring joint structure: player $i$ at slot $k$ updates the marginal estimate $\hat{w}^i(a^i)$ from the joint feedback. We do not include a separate centralized oracle because under Problem~A the decentralized \texttt{mCascadeUCB-A} already maintains identical statistics to a centralized UCB on joint arms; \texttt{mCascadeUCB-A} therefore serves simultaneously as our algorithm and as the centralized benchmark.

\paragraph{Setup.}
Attraction probabilities are drawn $w(\bm{e})\sim\mathrm{Uniform}[0.1, 0.9]$ and fixed across seeds. We average over $5$ random seeds; shaded regions show $\pm 2$ standard deviations. The instance is $L=3$, $K=2$, $M=3$ (27 joint arms) with horizon $T=5\times 10^4$. We sweep two termination regimes: \emph{low} ($v\in[0.15, 0.25]$) and \emph{high} ($v\in[0.85, 0.95]$).

\subsection{Cumulative Regret}

Figures~\ref{fig:regret-low} and~\ref{fig:regret-high} report cumulative regret.

\begin{figure}[t]
    \centering
    \includegraphics[width=\columnwidth]{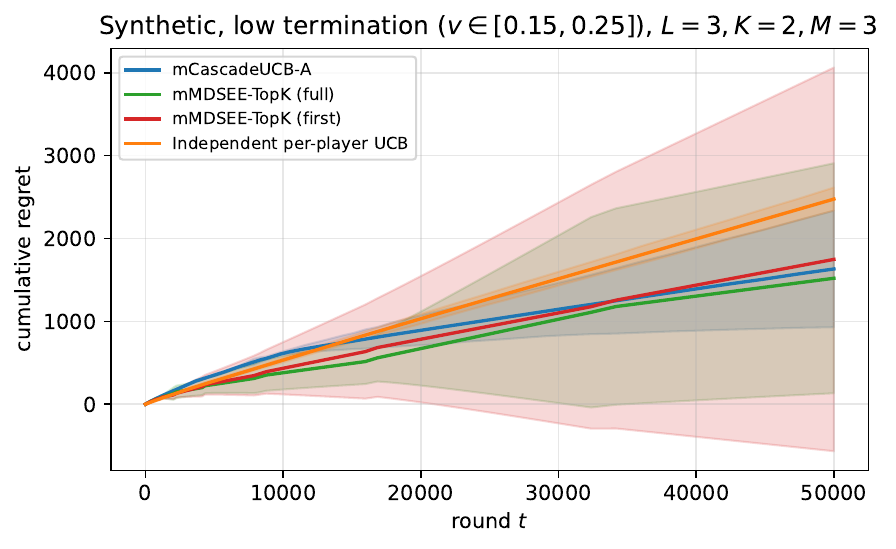}
    \caption{Cumulative regret under \emph{low} termination probability $v\in[0.15, 0.25]$. Coordinated methods cluster tightly; \texttt{mMDSEE-TopK} with full feedback benefits from observations in deeper slots.}
    \label{fig:regret-low}
\end{figure}

\begin{figure}[t]
    \centering
    \includegraphics[width=\columnwidth]{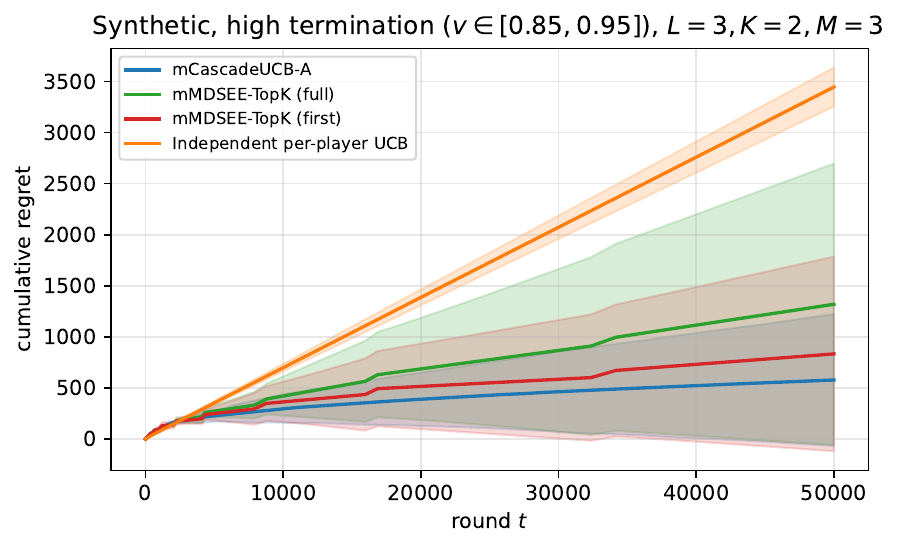}
    \caption{Cumulative regret under \emph{high} termination probability $v\in[0.85, 0.95]$. \texttt{mCascadeUCB-A} is best; first-slot \texttt{mMDSEE-TopK} outperforms full-slot here because deeper-slot observations are rare. Independent per-player UCB plateaus at $\sim 6\times$ the regret of coordinated methods.}
    \label{fig:regret-high}
\end{figure}

\paragraph{High termination ($v\in[0.85, 0.95]$).}
When users frequently terminate after the first click (Fig.~\ref{fig:regret-high}), most click feedback comes from slot~1. \texttt{mCascadeUCB-A} is the most data-efficient ($\approx 580$ at $T{=}5\times 10^4$). The first-slot variant of \texttt{mMDSEE-TopK} ($\approx 810$) actually \emph{outperforms} its full-slot counterpart ($\approx 1300$): when later-slot observations are rare and noisy, restricting updates to slot~1 reduces variance. Independent per-player UCB plateaus around $3{,}500$, an order of magnitude worse.

\paragraph{Low termination ($v\in[0.15, 0.25]$).}
When cascades typically reach both slots (Fig.~\ref{fig:regret-low}), full-feedback algorithms exploit every observation. \texttt{mMDSEE-TopK} with full feedback edges out its first-slot counterpart, consistent with the $\alpha = 1 + (K-1)p_{\min}$ scaling predicted by Theorem~\ref{thm:mDSEE-TopK-2}; the gap is modest at $K=2$ but is expected to widen as $K$ grows. Independent per-player UCB again pays a multiplicative penalty for ignoring joint structure.

\subsection{Discussion}

Three insights emerge:

\emph{Multi-slot feedback matters when $p_{\min}$ is large.} The $\alpha = 1 + (K-1)p_{\min}$ factor predicted by Theorem~\ref{thm:mDSEE-TopK-2} matches the observed gap between full- and first-slot variants at low termination.

\emph{First-slot feedback can win at high termination.} When most clicks terminate at slot~1, full-slot updates inject noise rather than information. This is consistent with our analysis: the $(K-1) p_{\min}$ improvement collapses to zero as $p_{\min} \to 0$.

\emph{Coordination beats independence even with asymmetric rewards.} Independent UCB consistently lags coordinated methods across both termination regimes, showing that the multiplayer asymmetric setting is genuinely harder than $M$ parallel single-agent problems and that joint-arm coordination---even without communication, as in \texttt{mMDSEE-TopK}---provides real value.

\paragraph{Computational cost.}
All algorithms have per-round cost dominated by the top-$K$ selection over $L^M$ joint arms, i.e.\ $O(L^M \log L^M)$. \texttt{mMDSEE-TopK} requires \emph{no} inter-player communication; \texttt{mCascadeUCB-Intervals-Ranking} requires action observability but no protocol messages. A single $T{=}5\times 10^4$ run completes in a few seconds on a standard laptop.

\section{Related Work}\label{sec:related}

\paragraph{Click-model bandits.}
The cascade model~\cite{kveton2015cascading} and its many extensions---combinatorial actions~\cite{kveton2015combinatorial}, contextual settings~\cite{pmlr-v48-lif16}, Thompson sampling~\cite{zhong2021thompson}, scalability improvements~\cite{zong2016cascading}, and clustering-based methods~\cite{li2018onlineclustering}---form the algorithmic foundation for learning to rank under bandit feedback. The dependent click model~\cite{guo2009efficient, katariya2016dcm} extends single-click cascades to multi-click sessions and is the closest single-agent precursor to our work; recent variants incorporate fatigue~\cite{cao2020fatigue}, exposure bias~\cite{mansoury2024mitigating}, RL-style ranking~\cite{du2024cascading}, and adversarial robustness~\cite{xie2025adversarialcascade}. All of these assume a single decision-maker.

\paragraph{Decentralized and multiplayer bandits.}
A parallel line of work studies multiple agents learning simultaneously under communication or observation constraints~\cite{awerbuch2005competitive, landgren2016distributed, cesabianchi2016delay, boursier2022survey, alatur2020adversarial, lugosi2021nocollision, branzei2021cooperation}. Most relevant here are decentralized learning with information asymmetry and no communication~\cite{chang2022online, chang2023optimal, chang2025multiplayerinformationasymmetriccontextual, chang2025metric} and federated/combinatorial multi-agent formulations~\cite{fourati2024combinatorial}, which provide the techniques we extend to multi-click feedback. Real-world recommender deployments~\cite{zou2025surveyrealworldrecommendersystems, wang2025recommending, loepp2023multi, fletcher2023recommender, demsynjones2022measurementapplicationspositionbias} and bandit models with structured feedback~\cite{alon2014graph, mo2023distributed} further motivate the decentralized cascading setting. To our knowledge no prior work has combined multi-click cascade feedback with multi-agent decentralized learning under information asymmetry.

\section{Conclusion}

We study decentralized multi-click cascading bandits with private observations and information asymmetry. Building on the Dependent Click Model and multiplayer bandit theory, we design algorithms that handle action asymmetry, reward asymmetry, and their combination. Our methods extend UCB-based approaches and MDSEE-style exploration to the multi-agent cascade setting, achieving sublinear regret under both first-slot and multi-slot feedback.

Our theoretical analysis captures the role of multiple clicks and slot termination probabilities. Empirical results confirm that leveraging multi-slot feedback significantly improves performance when termination probabilities are small, and that coordinated decentralized learning consistently outperforms independent per-player baselines.

\paragraph{Open problems.}
Two questions remain open. First, we prove no matching information-theoretic lower bound for any of Problems~A--C; establishing such bounds (or identifying a regime where our upper bounds are tight) is the natural next step. Second, the worst-case $L^M$ dependence reflects the absence of factored structure in our model; identifying mild parametric assumptions (e.g.\ generalized linear link functions over player features) under which polynomial-in-$L$ regret is attainable would substantially improve scalability.

\section*{Acknowledgment}
The authors would like to thank the reviewers for their valuable feedback.

\bibliographystyle{ieeetr}
\bibliography{references}

\appendix

\section{Important Lemmas}
\begin{lemma}\label{Corollary5.5}
    Assume that $X_i-\mu$ are independent, $\sigma$-subgaussian random variables. Then for any $\varepsilon \geq 0$,
\[
\mathbb{P}(\hat{\mu} \geq \mu+\varepsilon) \leq \exp \left(-\frac{n \varepsilon^2}{2 \sigma^2}\right), \quad \mathbb{P}(\hat{\mu} \leq \mu-\varepsilon) \leq \exp \left(-\frac{n \varepsilon^2}{2 \sigma^2}\right),
\]
where $\hat{\mu}=\frac{1}{n} \sum_{t=1}^n X_t$.
\end{lemma}

\section{Proof of Lemma~\ref{lemma:observations}}

\begin{proof}
Suppose item $e$ was placed in slot $k$ for the $i$-th time, and let $p_i^k$ be the probability that this item was observed. Let $X_i$ be the Bernoulli random variable for the clicks for the $i$-th placement. Then $X_i-p_i^k$ is a sequence of mean-$0$ independent $\frac{1}{2}$-subgaussian random variables, so by Lemma~\ref{Corollary5.5} we have
\[
\Pr\!\left[\frac{\sum_i X_i}{M} < \frac{\sum_i p_i^k}{M} -\epsilon \right]
  \le \exp\!\left( -2M\epsilon^2\right).
\]
Setting $\epsilon = \frac{\alpha \sum_i p_i^k}{M}$ for some small $\alpha > 0$, we conclude that
\[
\Pr\!\left[\frac{\sum_i X_i}{M} < (1-\alpha)\frac{\sum_i p_i^k}{M} \right]
  \le \exp\!\left( -2\alpha^2 p_{\min}^2 M\right).
\]
Setting the right-hand side to be upper bounded by $1/T$:
\[
  \exp\!\left( -2\alpha^2 p_{\min}^2 M\right) = \frac{1}{T}
  \quad\Longrightarrow\quad
  \alpha = \sqrt{\frac{\log T}{2p_{\min}^2 M}}.
\]
Therefore, with probability at least $1-\frac{1}{T}$,
\[
  \#\{\text{observations of }e\text{ in slot }k\}
  > \left(1-\sqrt{\frac{\log T}{2p_{\min}^2M}}\right)Mp_{\min}.
\]
\end{proof}

\section{Proof of Theorem~\ref{thm:mCascadeUCB-Intervals-Ranking}}

\begin{proof}
A necessary condition for item $e$ to still be in the candidate set is that $\operatorname{UCB}(e) > \operatorname{LCB}(\bm{e^\star})$. This implies
\begin{align}
   \hat{\mu}_e + \sqrt{\frac{4\log T}{n_e(t)}} &> 
   \hat{\mu}_{\bm{e^\star}} - \sqrt{\frac{4\log T}{n_{\bm{e^\star}}(t)}}, \\
   \implies \mu_e + 2\sqrt{\frac{4\log T}{n_e(t)}} &> 
   \mu^* - 2\sqrt{\frac{4\log T}{n_e(t)}} \quad \text{(good event)}, \\
   \implies 2\sqrt{\frac{4\log T}{n_e(t)}} &> 
   \Delta_{\bm{e^\star}, e} - 2\sqrt{\frac{4\log T}{n_{\bm{e^\star}}(t)}}.
\end{align}

Since items in the candidate set are pulled in a round-robin fashion, we have $|n_e(t) - n_{\bm{e^\star}}(t)| \leq 1$. This implies that when $n_e(t) \geq 2$, we also have $n_{\bm{e^\star}}(t) \geq n_e(t)/2$. Plugging this in:
\begin{align}
2\sqrt{\frac{4\log T}{n_e(t)}} &> 
\Delta_{\bm{e^\star}, e} - 2\sqrt{\frac{4\log T}{n_e(t)/2}}, \\
\sqrt{\frac{\log T}{n_e(t)}} \left( 4 + 4\sqrt{2} \right) &> 
\Delta_{\bm{e^\star}, e}, \\
\implies \quad n_e(t) &< \frac{(4 + 4\sqrt{2})^2 \log T}{\Delta_{\bm{e^\star}, e}^2}.
\end{align}

Now suppose the optimal joint items are $\bm{e}^\star_1, \ldots, \bm{e}^\star_K$, where $\bm{e}^\star_r$ is the item selected in the $r$-th elimination phase. The total number of pulls across all suboptimal items is bounded as
\[
\tau < \sum_{r=1}^{K}\;\sum_{\bm{e}\,:\,w(\bm{e}) < \mu^\star_r}
\frac{(4 + 4\sqrt{2})^2 \log T}{\Delta_{\bm{e}^\star_r,\,\bm{e}}^{\,2}}.
\]
Under the good event, the regret from $t = \tau$ to $T$ is zero, since only top items are pulled from that point onward.
\end{proof}

\section{Proof of Theorem~\ref{thm:mCascadeUCB-Intervals-Ranking-multiple}}

\begin{proof}
By the reasoning similar to Theorem~\ref{thm:mCascadeUCB-Intervals-Ranking}, item $e$ needs only $N_{\mathrm{elim}} = \frac{8\log T}{\Delta_K^2}+1$ \emph{observations} before it is eliminated (the additional $+1$ accounts for the sabotage round). Since the first slot always gets observed, the total observations satisfy
\begin{align*}
Y_{\rm total} &= M + \sum_{r=2}^K \#\{\text{obs in slot }r\} \\
&> M + (K\!-\!1)\!\left(1-\sqrt{\tfrac{\log T}{2p_{\min}^2 M}}\right)\!Mp_{\min}.
\end{align*}

We require $Y_{\rm total}\ge N_{\rm elim}$. Let $A := 1+(K-1)p_{\min}$, $B := (K-1)\sqrt{\frac{\log T}{2}}$, and $R := \frac{8\log T}{\Delta_K^2}+1$. Then $AM - B\sqrt{M} - R \ge 0$. Setting $x := \sqrt{M}\ge 0$ yields the quadratic $Ax^2 - Bx - R \ge 0$, solved by
\[
x \ge \frac{B + \sqrt{B^2+4AR}}{2A}.
\]
Squaring gives $M \ge \left(\frac{B + \sqrt{B^2+4AR}}{2A}\right)^2$.

The instantaneous regret satisfies $R_t \le \sum_{k=1}^K[w(\bm{e}_k^*)-w(\bm{e}_k^t)]$, so the total regret is
$R_T \le \sum_{e\,\text{suboptimal}} \Delta_K \cdot \E[\#\text{placements of }e]$,
yielding the stated bound.
\end{proof}

\section{Proof of Theorem~\ref{thm:mDSEE-TopK-1}}

\begin{proof}
Split the total regret into exploration and commitment phases: $R_T = R_{T,E} + R_{T,C}$.

\textbf{Exploration term.}
Let $\nu_t(e)$ be the number of times item $e$ has been placed first up to time $t$. Since exploration occurs at $t=2^0,2^1,\dots,2^{\lfloor\log_2 T\rfloor}$, we have $\nu_t(e) \le F(\lfloor\log_2 T\rfloor) \lceil\log_2 T\rceil$. Each pull incurs at most $L$ regret, so
\[
R_{T,E} \le L^M F(\lfloor\log_2 T\rfloor)\lceil\log_2 T\rceil.
\]

\textbf{Commitment term.}
Let $\varepsilon<\tfrac{1}{2}\min_{1\le i\le K}\min_{\bm{e}}\Delta_{i+1,i}$. Define the good event for player~$i$ at time $t$:
$G^i_t(e) = \{|\widehat w^i_{n_t(e)}(e)- w(e)|<\varepsilon\}$.
If $G_t=\bigcap_i G^i_t(e)$ holds, all players choose the optimal ranking, so $R_t=0$. Hence
\begin{align*}
R_{T,C} &= \sum_{t=1}^T [\E[R_t\mid G_t] P(G_t) + \E[R_t\mid G_t^c] P(G_t^c)] \\
&\le \sum_{t=1}^T P(G_t^c).
\end{align*}
By De Morgan and a union bound,
$R_{T,C} \le M\sum_{t=1}^T\sum_{e\in E} P(|\widehat w_{n_t(e)}(e)- w(e)|\ge\varepsilon)$.
Hoeffding's inequality gives
$P(|\widehat w_{n_t(e)}(e)- w(e)|\ge\varepsilon) \le 2\exp(-2n_t(e)\varepsilon^2)$.

Since $F(\lambda)=\lambda$, the total number of explorations up to phase $\lambda = \log_2 t$ is $n_t(\bm{e}) = 1+2+\cdots+F(\lambda) = \Omega((\log_2 t)^2/2)$. Therefore $F_0(t) := n_t(\bm{e})/\ln t \ge c\ln t$ for a constant $c>0$. The exponent becomes
$t^{-2F_0(t)\varepsilon^2} \le t^{-2c\varepsilon^2\ln t} = \exp(-2c\varepsilon^2(\ln t)^2)$.

Define the comparison integral $I(c,\varepsilon) = \int_1^\infty e^{-2c\varepsilon^2(\ln x)^2}\,dx$. Substituting $u=\ln x$ and completing the square:
\begin{align*}
I(c,\varepsilon) &= e^{\frac{1}{8c\varepsilon^2}} \int_{-\frac{1}{4c\varepsilon^2}}^{\infty} e^{-2c\varepsilon^2 v^2}\,dv \\
&= \frac{\sqrt{\pi}}{2\sqrt{2c}\,\varepsilon} e^{\frac{1}{8c\varepsilon^2}} \left[1 + \mathrm{erf}\!\left(\frac{1}{\sqrt{8c}\,\varepsilon}\right)\right].
\end{align*}

Since the summand is positive and decreasing, $\sum_{t=1}^{\infty} \exp(-2c\varepsilon^2(\ln t)^2) \le 1 + I(c,\varepsilon)$, hence
\begin{align*}
R_{T,C} &\lesssim \frac{M L^M \sqrt{\pi}}{\sqrt{2c}\,\varepsilon} \exp\!\left(\frac{1}{8c\varepsilon^2}\right) \\
&\quad\times \left[1 + \mathrm{erf}\!\left(\frac{1}{\sqrt{8c}\,\varepsilon}\right)\right].
\end{align*}
This completes the proof.
\end{proof}

\section{Proof of Theorem~\ref{thm:mDSEE-TopK-2}}

\begin{proof}
We follow the proof of Theorem~\ref{thm:mDSEE-TopK-1} with minor modifications. Under multiple-slot feedback, each time $e$ is placed we obtain additional observations from deeper slots. The effective observation mass is
\[
\tilde n_t(e) = n_t(e) + (K-1)\!\left(1-\sqrt{\frac{\log T}{2p_{\min}^2 n_t(e)}}\right)n_t(e)\,p_{\min}.
\]

\textbf{Concentration bound.}
Hoeffding's inequality applied to the empirical estimate gives
$\Pr(|\widehat w_{\tilde n_t(e)}(e)-w(e)|\ge \varepsilon) \le 2\exp(-2\tilde n_t(e)\varepsilon^2)$.

\textbf{Lower bound on effective observations.}
Since $n_t(e)\ge F_0(t)\log_2 t$, dropping lower-order terms, we simplify the dominant scaling as
$\tilde n_t(e) = \Theta(F_0(t)\log t \cdot [1 + (K-1)p_{\min}])$.

\textbf{Regret bound.}
The commitment regret becomes
\begin{align*}
R_{T,C} &\le 2M|E|\sum_{t=1}^{T} \exp\!\Big(-2 F_0(t)\log(t) \\
&\quad\times (1 + (K\!-\!1)p_{\min})\varepsilon^2\Big).
\end{align*}
Using $\alpha = 1+(K-1)p_{\min}$, this rewrites to $R_{T,C} \le 2ML^M \sum_{t=1}^{\infty} t^{-2\alpha F_0(t)\varepsilon^2}$. Multi-slot feedback increases the effective observation count by $\alpha$, yielding
\begin{align*}
R_{T,C} &\lesssim \frac{M L^M \sqrt{\pi}}{\sqrt{2c}\,\varepsilon\sqrt{\alpha}} \exp\!\left(\frac{1}{8c\alpha\varepsilon^2}\right) \\
&\quad\times \left[1+\mathrm{erf}\!\left(\frac{1}{\sqrt{8c\alpha}\,\varepsilon}\right)\right].
\end{align*}
\end{proof}

\end{document}